\documentclass[3p]{sig-alternate-per-modified}

\usepackage[utf8]{inputenc}
\usepackage[english]{babel}
\usepackage{enumerate}
\usepackage{graphicx}
\usepackage{hhline}
\usepackage{ dsfont }
\usepackage{float}
\usepackage{natbib}
\usepackage{graphicx}
\usepackage{subcaption}
\usepackage{amsmath,amsfonts,amssymb,bm}
\usepackage{amsthm}
\usepackage{xfrac} % door Jasper
\usepackage{scalefnt}

\usepackage{tikz}
\usepackage{pgfplots}
\pgfplotsset{compat=1.18}

\DeclareMathOperator{\diag}{diag}
\DeclareMathOperator{\PBE}{PBE}

\newtheorem{mythm}{Theorem}
\newtheorem{mylem}{Lemma}
\newtheorem{mycor}{Corollary}

\theoremstyle{definition}
\newtheorem{myremark}{Remark}

\begin{document}

% \begin{frontmatter}

\title{Convergence of off-policy TD(0) with linear function approximation for reversible Markov chains}

\author{Maik Overmars, Jasper Goseling, Richard J. Boucherie \\
\affaddr{Stochastic Operations Research, University of Twente} \\
 \normalsize m.g.overmars@utwente.nl, j.goseling@utwente.nl, r.j.boucherie@utwente.nl}

% % Author affiliation
% \affiliation{organization={University of Twente},%Department and Organization

% % Abstract

\maketitle

\begin{abstract}
We study the convergence of off-policy TD(0) with linear function approximation when used to approximate the expected discounted reward in a Markov chain. 
It is well known that the combination of off-policy learning and function approximation can lead to divergence of the algorithm.
Existing results for this setting modify the algorithm, for instance by reweighing the updates using importance sampling. This establishes convergence at the expense of additional complexity.
In contrast, our approach is to analyse the standard algorithm, but to restrict our attention to the class of reversible Markov chains. We demonstrate convergence under this mild reversibility condition on the structure of the chain, which in many applications can be assumed using domain knowledge.
In particular, we establish a convergence guarantee under an upper bound on the discount factor in terms of the difference between the on-policy and off-policy process. 
This improves upon known results in the literature that state that convergence holds for a sufficiently small discount factor by establishing an explicit bound.
Convergence is with probability one and achieves projected Bellman error equal to zero. To obtain these results, we adapt the stochastic approximation framework that was used by~\citet{tsitsiklis1997analysis} for the on-policy case, to the off-policy case.
We illustrate our results using different types of reversible Markov chains, such as one-dimensional random walks and random walks on a weighted graph.
\end{abstract}

% %% Keywords
% \begin{keywords}
% %% keywords here, in the form: keyword \sep keyword
% Temporal-difference learning, Off-policy learning, Linear \\function approximation, Convergence, Reversible Markov chains
% %% PACS codes here, in the form: \PACS code \sep code

% \end{keywords}

% \end{frontmatter}

\section{Introduction}
Temporal-difference learning (TD learning)~\citep{sutton1988learning} is a reinforcement learning algorithm that aims to predict the expected reward of a dynamical system.
The convergence of TD learning has been well studied, for instance by considering different ways to approximate the value function that represents the predicted reward, and by using either on-policy or off-policy sampling of states.
We know that in the tabular case, on-policy TD learning converges to the optimal prediction~\citep{sutton1988learning,dayan1994td,jaakkola1994convergence}.
Under linear value function approximation, on-policy TD learning algorithm still converges. The point of convergence is no longer optimal, but does have a projected Bellman error (PBE) of zero~\citep{tsitsiklis1997analysis,dayan1992convergence,tsitsiklis1994asynchronous,tadic2001convergence}.
However, off-policy learning with function approximation may diverge, as shown through examples by~\citet{baird1995residual} and \citet{tsitsiklis1994asynchronous}.

All of the above work considers the convergence of TD learning for a general input, i.e., for any possible Markov chain.
In this work, by restricting our attention to the class of reversible Markov chains, we are able to prove convergence of the off-policy TD(0) algorithm under linear function approximation. 
The class of reversible Markov chains covers many application areas, such as queueing theory~\citep{allen2014probability}, statistical mechanics~\citep{davidson2013statistical}, and Bayesian statistics~\citep{bolstad2016introduction}.  
We show that off-policy TD(0) with linear function approximation converges whenever the on-policy and off-policy chain are reversible and have the same structure, under a condition on the discount factor.
In particular, the difference between the on-policy and off-policy chain imposes an upper bound on the discount factor.
Finally, we show in Remark \ref{rem:reversible} that our results also apply to a more general setting, where the reversibility restriction is not necessary.

The stability of TD(0) learning is commonly analysed by considering the update rule in general form
\begin{equation}
    w_{t+1} = w_t + \alpha_t (A(X_t)w_t + b(X_t)),
\end{equation}
where $X_t$ is a Markov chain describing the sampling process, $w_t$ is the parameter vector that is learned, $\alpha_t$ is the step size, and $A(X_t)$ and $b(X_t)$ contain the details of the update rule.
Also, define matrix $A$ as $A(X_t)$ for $X_t$ in steady-state.
The core requirement for convergence is that matrix $A$ is negative definite~\citep{tsitsiklis1997analysis,dayan1992convergence,sutton2016emphatic,sutton2018reinforcement}.
Matrix $A$ is negative definite if and only if $x^T A x < 0$ for all $x \neq 0$.
\citet{tsitsiklis1997analysis} showed that this matrix is always negative definite for on-policy TD learning with linear function approximation.
We generalise their result by showing that for reversible Markov chains, matrix $A$ is negative definite for off-policy TD learning with linear function approximation.
We show that $A=\Phi^T(\gamma \hat{Q}P - \hat{Q})\Phi$, where the elements of this matrix are the linear features $\Phi$, discount factor $\gamma$, off-policy distribution $\hat{Q}$, and the transition probabilities $P$ corresponding to the original chain.
If the on-policy and off-policy processes are very different, then the difference between $\hat{Q}P$ and $\hat{Q}$ can be large, and \(A\) may not be negative definite.
Decreasing the discount factor $\gamma$ can mitigate this effect, since $A$ always becomes negative definite when $\gamma$ approaches zero.
In this paper we address the open question as to how small the discount factor should be to achieve convergence.
% Although $A$ always becomes negative definite for small enough $\gamma$, we establish an exact upper bound on $\gamma$ that depends on the size of the perturbation.

Our main contribution is that we establish convergence with probability one of off-policy TD(0) with linear function approximation for the class of reversible Markov chains.
Our analysis builds on the work by \citet{tsitsiklis1997analysis}, and extends the convergence of on-policy learning to the off-policy case.
We model the Markov chain of the off-policy process as a perturbation of the original chain, where this perturbation describes the difference between $\hat{Q}P$ and $\hat{Q}$.
We derive an exact upper bound on the discount factor $\gamma$ in terms of this perturbation, and we show that the algorithm converges when this upper bound is satisfied. 
Larger perturbations imply that the discount factor should be chosen small, while a large discount factor limits the size of the perturbation.
The point of convergence is the fixed point of the projected Bellman operator, and this point has a projected Bellman error (PBE) of zero.
Finally, we show how to apply these results to common classes of reversible Markov chains, such as one-dimensional random walks, and random walks on a weighted graph.
 
In the literature other approaches have been introduced to deal with the divergence of off-policy TD learning. Some convergence results are available for the off-policy linear function approximation case, but these rely on modifying the algorithm.
Following his example of divergence,~\citet{baird1995residual} suggested residual algorithms, but these are slow to converge and require double samples. Similarly, least-squares temporal-difference (LSTD) as proposed by \citet{nedic2003least} is stable but has quadratic per-step complexity in the number of parameters.
A similar approach is used by gradient TD algorithms, like gradient temporal-difference (GTD) methods by~\citet{sutton2008convergent} and linear TD with gradient correction (TDC) by~\citet{sutton2009fast}, which perform true gradient descent in the Bellman error. 
A generalization of these algorithms to non-linear function approximation was presented by~\citet{maei2009convergent}.
These algorithms converge and do not encounter the double sample issue, but they require a second set of parameters with a second step size, and thus are more complex to implement.
Importance sampling techniques~\citep{sutton2016emphatic,precup2001off} have also been used to stabilise the algorithm by reweighing the updates, but these methods can have high variance depending on the setting. 
On the other hand, standard TD learning works well in many settings, has low complexity, and is easy to implement. Based on the current work, if reversibility can be assumed, the standard algorithm is guaranteed to converge. 
Also, our models and analysis opens the way to obtain bounds on the convergence rate similarly to~\citet{srikant2019finite}.

In Section \ref{sec:prelims} we specify the Markov chain model that is studied and the TD learning algorithm including linear function approximation and off-policy learning.
Section \ref{sec:main_result} contains the main convergence result of this paper.
In Section \ref{sec:example_processes} we apply the result to two types of reversible Markov chains: one-dimensional random walks and random walks on a weighted graph.
Finally, Section \ref{sec:conclusion} contains the conclusion, discussion and outlook for future work.

\section{Model formulation} \label{sec:prelims}
In this section we specify the Markov chain model that we study as well as off-policy TD learning with linear function approximation. Moreover, we connect this to a stochastic approximation result and define the matrix $A$, which we will use in our analysis to establish convergence conditions. Regarding notation: Unless otherwise specified, all vectors are column vectors, and we will interchangeably use subscripts and parentheses to denote elements of vectors. Also, \(\|v\|\) denotes the Euclidean norm of vector \(v\) and \(\|A\|\) denotes the corresponding induced norm for matrix \(A\).

We consider a \emph{reversible}, irreducible and aperiodic Markov chain $\{S_t \in \mathcal{S} \ | \ t=0,1,\ldots\}$ with finite state space $\mathcal{S}=\{1,\ldots,n\}$~\citep{aldous-fill-2014}.
The transition probability matrix is denoted by $P\in \mathbb{R}^{n\times n}$, where entry $p_{ij}$ gives the probability of a transition from state $i$ to state $j$. Let $q\in \mathbb{R}^n$ be the row vector that denotes the unique steady-state probability distribution satisfying $q=qP$. Let $Q=\diag(q)$, the diagonal matrix with the elements of $q$ on the diagonal. 
Markov chain $S_t$ is reversible if and only if $q_i p_{ij} = q_j p_{ji}$ for any pair of states $i,j\in\mathcal{S}$.
This implies that $p_{ij}>0\Longleftrightarrow p_{ji}>0$, such that we can define the \textit{transition ratio}
\begin{equation} \label{eq:rate}
    \rho_{ij} := \frac{p_{ij}}{p_{ji}}.
\end{equation}
for any $i,j\in S$.
As a result, $\rho_{ji}=\rho_{ij}^{-1}$ exists and we can write
\begin{equation} \label{eq:distr_rate}
    q_j = q_i \rho_{ij}.
\end{equation}

We denote by $r\in \mathbb{R}^n$ the reward vector over the states, where $r(i)$ gives the reward for state $i$.
% The rewards are discounted by discount factor $\gamma \in (0,1)$.
Our interest is in the expected sum of discounted rewards, defined through the value function $V: \mathcal{S}\xrightarrow{}\mathbb{R}$, with $V(i)$ the value when starting in state $i$, i.e.,
\begin{equation} \label{eq:value_func}
    V(i) := \mathbb{E}\left[\sum_{t=0}^\infty \gamma^t r(S_t) | S_0 = i\right],
\end{equation}
where $\gamma \in (0,1)$ is the discount factor. We aim to approximate \(V\) using linear function approximation. In particular, let
\begin{equation}
    V_w(i) := \sum_{k=1}^K w(k) \phi_k(i),
\end{equation}
where $\phi_k\in\mathbb{R}^n$ for $k=1,\ldots,K$ are the feature vectors, $w\in \mathbb{R}^K$ are the weights, and $i\in\mathcal{S}$. We define $\Phi\in \mathbb{R}^{n \times K}$ as the matrix where column $k$ corresponds to vector $\phi_k$, such that we can write
\begin{equation}
    V_w = \Phi w.
\end{equation}
We assume that feature matrix $\Phi$ has full column rank, i.e., the basis vectors $\phi_k$, $k=1,\ldots,K$, are linearly independent. 
We will use the notation $\phi(s) = (\phi_1(s), \phi_2(s), \ldots, \phi_K(s))$, \(s\in\mathcal{S}\), i.e., \(\phi(s)\) is the transpose of a row in \(\Phi\).

Next, we introduce the off-policy TD(0) algorithm. This algorithm iteratively updates the weights in 
\(w\) with the aim of approximating $V$ by \(V_w\). The update rule is
\begin{equation} \label{eq:td_update}
    w_{t+1} = w_t + \alpha_t d_t \phi(\hat{S}_t),
\end{equation}
where \(\alpha_t>0\) is the step size, \(d_t\) is the temporal difference which will be specified below, and \(\hat{S}_t\in\mathcal{S}\) is a state that is sampled at time \(t\). The off-policy property means that the states \(\hat{S}_t\) are sampled from a different Markov chain than $\{S_t \in \mathcal{S} \ | \ t=0,1,\ldots\}$, the Markov chain that the algorithm aims to estimate \(V\) for. In particular, we have reversible, irreducible and aperiodic Markov chain $\{\hat S_t \in \mathcal{S} \ | \ t=0,1,\ldots\}$ with transition probability matrix $\hat{P}$, steady-state probability distribution $\hat{q}$ and $\hat{Q}=\diag(\hat{q})$. We assume that
\begin{equation}
    \hat{p}_{ij}>0 \Longleftrightarrow p_{ij}>0,
\end{equation}
such that both Markov chains have the same structure. We will refer to $\{S_t \in \mathcal{S} \ | \ t=0,1,\ldots\}$ and $\{\hat S_t \in \mathcal{S} \ | \ t=0,1,\ldots\}$ as the original and the perturbed Markov chain, respectively. The temporal difference \(d_t\) is based on a one-step transition in the original Markov chain when starting from state \(\hat S_t\), i.e. let \(S_t'\) be defined as
\begin{equation}
    \mathbb{P}(S_t'=j|\hat{S}_t=i) = p_{ij}.
\end{equation} 
The temporal difference is then defined as
\begin{align}
    d_t &:= r(\hat{S}_t) + \gamma V_{w_t}(S_t') - V_{w_t}(\hat{S}_t), \\
    &= r(\hat{S}_t) + \gamma \phi^T(S_t') w_t - \phi^T(\hat{S}_t) w_t.
\end{align}

We assume that the step sizes $\alpha_t$ are positive, nonincreasing, and satisfy the Robbins-Monro conditions~\citep{robbins1951stochastic}
\begin{align} \label{eq:step_size}
    &\sum_{t=0}^\infty \alpha_t = \infty, \\
    &\sum_{t=0}^\infty \alpha^2_t < \infty.
\end{align}

By rewriting~\eqref{eq:td_update} as
\begin{equation} \label{eq:td_update_in_A_b}
    w_{t+1} = w_t + \alpha_t \left(A(\hat S_t, S_t') w_t + b(\hat S_t, S_t') \right),
\end{equation}
with
\begin{equation} \label{eq:transient_a}
    A(\hat S_t, S_t') := \phi(\hat{S}_t)\left(\gamma \phi^T(S_t') - \phi^T(\hat{S}_t) \right),
\end{equation}
and
\begin{equation} \label{eq:transient_b}
    b(\hat S_t, S_t') := \phi(\hat{S}_t)r(\hat{S}_t),
\end{equation}
we see that it is an instance of a stochastic approximation algorithm that is well studied. In particular, we have the following convergence result.
\begin{mythm} [Theorem 17 from \citet{benveniste2012adaptive}] \label{thm:stochastic_approx}
    Consider an iterative algorithm of the form
    \begin{equation}
        w_{t+1} = w_t + \alpha_t (A(X_t)w_t + b(X_t)),
    \end{equation}
    where
    \begin{enumerate}
        \item the step size sequence $\alpha_t$ is positive, nonincreasing, and satisfies $\sum_{t=0}^\infty \alpha_t = \infty$ and $\sum_{t=0}^\infty \alpha_t^2 < \infty$;
        \item $X_t$ is a Markov process with unique steady-state distribution, and $X$ denotes the random variable distributed according to this steady-state distribution. 
        Furthermore, there exists a mapping $h:\mathcal{S}\xrightarrow{}\mathbb{R}$ satisfying the remaining conditions;
        \item $A(\cdot)$ and $b(\cdot)$ are matrix and vector valued functions, for which $A=\mathbb{E}[A(X)]$ and $b=\mathbb{E}[b(X)]$ are well defined and finite;
        \item matrix $A$ is negative definite;
        \item there exists constants $M$ and $g$ such that for all $x$
        \begin{equation}
            \sum_{t=0}^\infty ||\mathbb{E}[A(X_t)|X_0=x)-A|| \leq M (1+h^g(x)),
        \end{equation}
        and
        \begin{equation}
            \sum_{t=0}^\infty ||\mathbb{E}[b(X_t)|X_0=x)-b|| \leq M (1+h^g(x));
        \end{equation}
        \item for any $g>1$, there exists a constant $\mu_g$ such that for all $X, t$
        \begin{equation}
            \mathbb{E}[h^g(X_t)|X_0=x] \leq \mu_g (1+h^g(x)).
        \end{equation}
    \end{enumerate}
    Then, $w_t$ converges to $w^*$, where $w^*$ is the unique vector that satisfies $Aw^* + b = 0$.
\end{mythm}
Our statement of the above theorem is in fact, a special case of the more general Theorem 17 from~\citet{benveniste2012adaptive}. The above form appears in~\citet[Theorem 2]{tsitsiklis1997analysis}. Most of our analysis will deal with verifying the conditions of this theorem for our setting in which \(X_t=(\hat S_t, S_t')\). To set the stage for this, we define matrix $A$ and vector \(b\) as
\begin{equation} \label{eq:Ab_def}
    A := \Phi^T(\gamma\hat{Q}P - \hat{Q})\Phi,\quad\text{and}\quad b := \Phi^T \hat{Q} r.
\end{equation}
Crucial parts of our analysis will be to show that \(A=\mathbb{E}[A(\hat S, S')]\) and \(b=\mathbb{E}[b(\hat S, S')]\),
and to derive conditions under which \(A\) is negative definite, henceforth abbreviated as n.d.

In addition to convergence guarantees, we study the point of convergence \(w^*\). In particular, we will show that \(w^*\) minimises the projected Bellman error (PBE)~\citep{sutton2018reinforcement}, which we define next. In order to do so, consider the inner product $\langle x, y \rangle_{\hat{Q}} := x^T \hat{Q} y$, and the associated weighted Euclidean norm $||\cdot||_{\hat{Q}}:=\sqrt{\langle \cdot, \cdot \rangle_{\hat{Q}}}$ on the space $\{\Phi w | w\in\mathbb{R}^K\}$ of value function approximations given linear features $\Phi$. Now,
\begin{equation} \label{eq:pbe}
    \text{PBE}(w) := ||\Pi_{\hat{Q}} \mathcal{B}\left(\Phi w\right) - \Phi w||_{\hat{Q}},
\end{equation}
where \(\mathcal{B}\) is the Bellman operator that satisfies $V=\mathcal{B}V$, and \(\Pi_{\hat{Q}}\) is the projection operator from \(\mathbb{R}^n\) onto \(\{\Phi w | w\in\mathbb{R}^K\}\) with respect to the norm \(||\cdot||_{\hat{Q}}\).

\section{Main results} \label{sec:main_result}
In this section we state the main result of this paper, that is the convergence of temporal-difference learning applied on reversible Markov chains with linear function approximation and off-policy learning.
\begin{mythm} \label{thm:main}
    Let $c \geq 1$ satisfy
    \begin{equation} \label{eq:perturbation1}
        \frac{1}{c} \rho_{ij} \leq \hat{\rho}_{ij} \leq  c \rho_{ij},
    \end{equation}
    for all $i,j\in\mathcal{S}$ for which \(p_{ij}>0\), and let
    \begin{equation} \label{eq:gamma_cond1}
        \gamma < \frac{2}{c+1}.
    \end{equation}
    Then $w_t$ converges to $w^*$ with probability 1, where $w^*$ is the unique vector that satisfies \(\PBE(w^*)=0\).
\end{mythm}

\begin{myremark} \label{rem:reversible}
    Careful analysis of the proof of Theorem \ref{thm:main} shows that we do not need reversibility of the original Markov chain.
    % the result also holds when only the perturbed chain is reversible.
    In fact, the original chain only needs to satisfy $p_{ij} > 0$ iff $p_{ji} > 0$ for any states $i$ and $j$, such that $\rho_{ij}$ exists.
    % and we can apply \eqref{eq:perturbation1}.
    Thus, for any transition in the original Markov chain the reverse transition probability should also be positive. 
\end{myremark}

The factor $c$ determines the size of the perturbation, in other words how different the off-policy distribution is from the on-policy distribution.
The proof of Theorem~\ref{thm:main} is provided in Appendix \ref{app:proofthmmain}.
We do, however, present the main technical lemma that is used in the proof here. In particular, we give a sufficient condition for $A$ as defined in~\eqref{eq:Ab_def} to be n.d.
\begin{mylem} \label{lem:diag_dom}
    If 
    \begin{equation} \label{eq:diag_dom_condition}
        \gamma < \frac{2}{1+\frac{1}{\hat{q}_i}\sum_{j=1}^n \hat{q}_j p_{ji}}
    \end{equation}
    for all $i\in\mathcal{S}$, then matrix $A$ is n.d.
\end{mylem}
The proof of the lemma is given in Appendix~\ref{app:proof_diag_dom} and is based on establishing diagonal dominance, see for instance \citet[Corollary 7.2.3]{horn1985matrix}.
Note, that this is a sufficient but not necessary condition for $A$ to be n.d., as a matrix being n.d. does not imply that it is also diagonally dominant. To illustrate~\eqref{eq:diag_dom_condition} consider the on-policy case where $\hat{Q}=Q$. There we have $\hat{q}_i = \sum_{j=1}^n \hat{q}_j p_{ji}$, and we can rewrite \eqref{eq:diag_dom_condition} as
\begin{align}
    \gamma < \frac{2}{1+\frac{1}{\hat{q}_i}\sum_{j=1}^n \hat{q}_j p_{ji}}
    = \frac{2}{1+\frac{1}{\hat{q}_i}\hat{q}_i}
    = 1,
\end{align}
which is always satisfied as $\gamma \in (0,1)$. Hence, we establish the result of \citet{tsitsiklis1997analysis} result that on-policy TD(0) learning with linear function approximation converges as a special case of Theorem~\ref{thm:main}. 

Though it is not stated explicitly, Lemma~\ref{lem:diag_dom} and its proof do not require reversibility of either the original or the perturbed Markov chain. 
However, Lemma~\ref{lem:diag_dom} seems to be most useful in the case of reversible Markov chains, where limited information about the original and the perturbed chain may provide convergence guarantees. 
One instance of this is Theorem~\ref{thm:main}, where we only require a bound the perturbation of the transition ratios. 
In Section~\ref{sec:example_processes} we give another application of Lemma~\ref{lem:diag_dom} that demonstrates that if more information about the process is available, tighter bounds on \(\gamma\) can be obtained.

\section{Applications} \label{sec:example_processes}
In this section, we first show how to apply Theorem \ref{thm:main} to the class of random walks on a weighted graph in the case when a bound on the perturbations of the weights is available. Next, we apply Theorem \ref{thm:main} to the class of one-dimensional random walks. Furthermore, using Lemma~\ref{lem:diag_dom} we show how to obtain better bounds if more information about the process is known.

\subsection{Random walk on a weighted graph}
Consider a random walk on a weighted graph \citep{aldous-fill-2014}.
We define the undirected, complete graph $G=(V,E)$ with vertices $V$ and edges $E$. Edge $\{i,j\}\in E$ for $i,j\in V$ has weight $u_{ij}>0$. Since edges are undirected we have \(u_{ij}=u_{ji}\).
We then consider the Markov chain with transition probabilities defined as
\begin{equation} \label{eq:graph_probs}
    p_{ij} = \frac{u_{ij}}{\sum_{k=1}^n u_{ik}}.
\end{equation}
The equilibrium distribution for this process equals
\begin{equation} \label{eq:graph_distr}
    q_i = \frac{\sum_{k=1}^n u_{ik}}{\sum_{\ell=1}^n \sum_{k=1}^n u_{\ell k}}.
\end{equation}
It can be easily shown that this process is reversible. 
Furthermore, every reversible Markov chain can be rewritten as a random walk on a weighted graph \citep{aldous-fill-2014}.
This is done by setting the weights as $u_{ij}=q_i p_{ij}$, so that detailed balance immediately holds using $u_{ij}=u_{ji}$.
In this way, we can apply our result to any reversible Markov chain by bounding the weights in a similar manner as shown below.
However, this requires knowledge of stationary distributions $q$ and $\hat{q}$ in order to define the weights.
As a result, it is easier for a general Markov chain to apply Theorem \ref{thm:main} by bounding the ratios of transition probabilities directly.

Now consider a perturbed process defined on the same graph \(G\) with weights \(\hat u_{ij}>0\). 
Moreover, suppose that the difference in weights between the original and the perturbed process can be bounded as
\begin{equation}
    \frac{1}{\delta}\leq\frac{\hat{u}_{ij}}{u_{ij}}\leq \delta,
\end{equation}
for all edges \(\{i,j\}\in E\), for some \(\delta\geq 1\). Using \(\hat u_{ij} = \hat u_{ji}\) and \(u_{ij} = u_{ji}\) we can write
\begin{align}
    \frac{\hat{p}_{ij}}{\hat p_{ji}}
        &= \frac{\sum_{k=1}^n \hat u_{jk}}{\sum_{k=1}^n \hat u_{ik}} \\
        &\leq \frac{\delta\sum_{k=1}^n u_{jk}}{\delta^{-1}\sum_{k=1}^n u_{ik}} \\
        &= \delta^2 \frac{{p}_{ij}}{p_{ji}}.
\end{align}
A similar argument can be made for the other direction, so that we have
\begin{equation}
    \frac{1}{\delta^2} \leq \frac{\hat{\rho}_{ij}}{\rho_{ij}} \leq \delta^2.    
\end{equation}
Hence, we can apply Theorem~\ref{thm:main} with \(c=\delta^2\). 
\begin{figure}
    \centering
    % \resizebox{\linewidth}{!}{%
    % \tikzsetnextfilename{tikzfig-numerical_results}
% \tikzpicturedependsonfile{./numerical_results.csv}
\begin{tikzpicture}
    \begin{semilogxaxis}[
      width=9cm,
    height=5cm,
      %xmode = log
      xlabel=\(\delta\),
      ylabel=\(\gamma\),
      font=\scriptsize,
      legend columns = 2,
      legend style={
            at={(0.5,-0.3)},
            anchor=north,
           font=\scriptsize,
      },
    legend image post style={mark indices={}} % To get legend markers when using 'mark indices'
    ]

    %%%%%%%%%%%%%%%%%
    %
    %
    %
    %%%%%%%%%%%%%%%%%

    \addplot[
        color=black,
        dashed,
        line width = 0.3mm,
        mark=o,
        mark options={solid},
        mark indices={5,15,25,35}
        ]
      table[
        x index=0, y index=5, col sep=comma, header=false
        ]
      {./numerical_results.csv};

    \addplot[
        color=black,
        dotted,
        line width = 0.3mm,
        mark=square,
        mark options={solid},
        mark indices={5,15,25,35}
        ]
      table[
        x index=0, y index=4, col sep=comma, header=false
        ]
      {./numerical_results.csv};

      \addplot[
        color=red,
        line width = 0.3mm,
        mark=square*,
        mark indices={5,15,25,35}
        ]
      table[
        x index=0, y index=2, col sep=comma, header=false
        ]
      {./numerical_results.csv};

    \addplot[
      color=blue,
      line width = 0.3mm,
      mark=*,
      mark indices={5,15,25,35}
      ]
    table[
      x index=0, y index=3, col sep=comma, header=false
      ]
    {./numerical_results.csv};

    \addplot[
        color=black,
        line width = 0.3mm,
        mark=triangle*,
        mark indices={10,30}
        % mark repeat = 9,
        % mark phase = 8
        ]
      table[
        x index=0, y index=1, col sep=comma, header=false
        ]
      {./numerical_results.csv};
    
    \legend{
        Cor.~\ref{cor:improved_bound} (\(\rho\rightarrow\infty\)),
        Cor.~\ref{cor:improved_bound} (\(\rho\downarrow 0\)),
        Cor.~\ref{cor:improved_bound} (\(\rho=2\)),
        Cor.~\ref{cor:improved_bound} (\(\rho=0.5\)),
        Thm.~\ref{thm:main}
    };

    \end{semilogxaxis}
    \end{tikzpicture}
    % }
    \caption{Upper bounds on $\gamma$ given by Theorem \ref{thm:main} and Corollary \ref{cor:improved_bound} for the simple random walk.} %  The x-axis has a logarithmic scale.
    \label{fig:bound_comparison}
\end{figure}
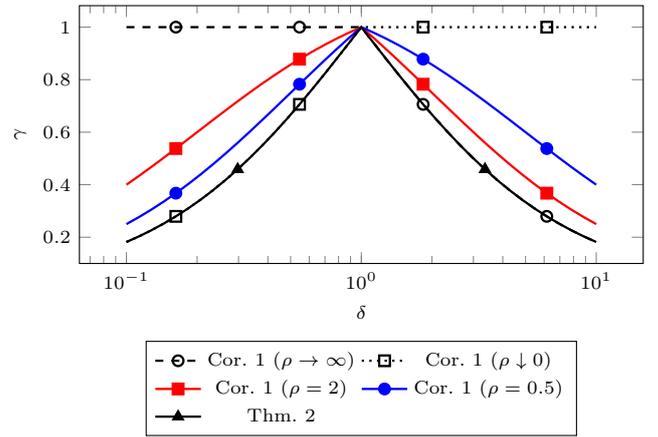

\subsection{One-dimensional random walk} \label{sec:oned_randomwalk}
Consider the case that the original and the perturbed chain are one-dimensional random walks on $n$ states, i.e.,
\(p_{ij}>0\) and \(\hat p_{ij}>0\) if and only if i) \(|i-j|=1\), or ii) \(i=j=1\), or iii) \(i=j=n\).
For these random walks it can be easily verified that detailed balance holds.
As a result, they are reversible and the stationary distributions can be shown to equal
\begin{equation} \label{eq:oned_steadystate}
    q_i = q_1 \prod_{j=1}^{i-1} \rho_{j, j+1}
\end{equation}
for $i=2,\ldots,n$, where \(q_1\) follows from normalization and where a similar expression holds for \(\hat q_i\).
Note, that the transitions ratios capture the drift of the random walk to the left or right side of the state space. A transition ratio \(\rho_{i,i+1}>1\) corresponds to a drift to the right side, while a transition ratio \(\rho_{i,i+1}<1\) corresponds to a drift to the left side.

Let
\begin{equation}
    c_{ij} = \begin{cases}
        {\displaystyle\frac{\hat{\rho}_{ij}}{\rho_{ij}}} &\mbox{ if } p_{ij}, \hat{p}_{ij}>0, \\
        1 & \mbox{ otherwise,}
    \end{cases}
\end{equation}
for $i,j\in\{1,\ldots,n\}$ and
\begin{equation}
    c = \max_{i,j\in\{1,\ldots,n\}}\left\{\max\left\{c_{ij}, \frac{1}{c_{ij}}\right\}\right\}.
\end{equation}
The interpretation of \(c\) is that it is the maximum relative difference in drift between the original and the perturbed chain. From Theorem \ref{thm:main} it follows that as long as our feature matrix is full rank and $\gamma < 2/(c+1)$, TD(0) converges.

Next, we demonstrate how Lemma~\ref{lem:diag_dom} can be used to obtain a better bound on \(\gamma\) if more information about the original and perturbed process are available. In particular, we restrict our attention to the simple random walk in which the transition ratios are constant and known for all states.
\begin{mycor} \label{cor:improved_bound}
    Let the original and the perturbed process be simple random walks on $n$ states with
    $\rho_{i,i+1} = \rho$ and $\hat{\rho}_{i,i+1} = \hat{\rho}$ for $i=1,\ldots,n-1$. Let $\delta = \hat{\rho}/\rho$ and
    \begin{equation} \label{eq:improved_bound}
        \gamma < \min\left\{\frac{2(\rho+1)}{\rho + 2 + \delta\rho}, \frac{2(\rho+1)}{2\rho+1+\delta^{-1}}\right\}.
    \end{equation}
    Then $w_t$ converges to $w^*$ with probability 1, where $w^*$ is the unique vector that satisfies \(\PBE(w^*)=0\).
\end{mycor}
The proof of the corollary is given in Appendix~\ref{app:proof_cor_improved_bound}.
The difference with the proof of Theorem~\ref{thm:main}, in which no structural information besides reversibility is available, is that for this corollary 
we use the random walk structure of the process. The key idea is that Lemma~\ref{lem:diag_dom} imposes a condition for each state and that it turns out that the conditions imposed by boundary states 1 and $n$ are the most stringent. Moreover, these conditions can be loosened w.r.t.\ the bound given in Theorem~\ref{thm:main}, most importantly because \(\hat{\rho}/\rho\) is now known exactly.

We have illustrated the bounds of Corollary~\ref{cor:improved_bound} in Figure~\ref{fig:bound_comparison} for \(\rho=2\) and \(\rho=0.5\) and the limiting behavior for \(\delta\to\infty\) and \(\delta\downarrow 0\). For comparison, we have also depicted the bound of Theorem~\ref{thm:main} by using \(c=\delta^{-1}\) for \(\delta<1\). The figure illustrates that Lemma~\ref{lem:diag_dom} allows for better bounds than given in Theorem~\ref{thm:main} by using more information. If such information is available depends on the application. The value of Theorem~\ref{thm:main} is that it only requires a bound on the difference between the original and the perturbed process.

\section{Conclusion and discussion} \label{sec:conclusion}
In this paper, we have shown convergence of off-policy temporal-difference learning with linear function approximation for reversible Markov chains.
Unlike previous research, which modifies the temporal-difference learning algorithm to establish convergence results, we instead consider the base algorithm.
Key in our analysis is the matrix $A=\Phi^T (\gamma \hat{Q}P - \hat{Q})\Phi$, which is required to be negative definite for the algorithm to converge.
In our setting, we show that this matrix is negative definite under a condition between the discount factor $\gamma$ and the size of the perturbation caused by off-policy learning.

Our work may be extended to more general settings by eliminating various assumptions at the cost of complexity of the proof. 
Think of the assumption of a finite state space, or linear independence of the basis functions, or aperiodicity of the underlying Markov chains, as mentioned by~\citet{tsitsiklis1997analysis} for the on-policy case.

A central assumption in this paper is reversibility of the Markov chains, and extending our results beyond this assumption would improve their applicability.
Reversibility is a technical assumption in our analysis, as it allows us to manipulate the sum $\sum_{j=1}^n \hat{q}_j p_{ji}$ found in Lemma \ref{lem:diag_dom}.
This sum captures the size of the off-policy behaviour in our model.
In Remark \ref{rem:reversible}, we have already shown that our results also hold if the original Markov chain is not reversible, but still satisfies $p_{ij}>0 \Longleftrightarrow p_{ji}>0$ for any $i,j\in\mathcal{S}$, i.e., the reverse transition always exists.
It is worth investigating to extend our result to the case where the perturbed chain is also not reversible, and instead satisfies the assumption that the reverse transition exists.
This extension requires changing our analysis and it seems to lead to a more complex result as we can no longer use the definition of reversibility to rewrite the terms in the sum mentioned above. 
As a result, we would have to add correction terms to account for this, which may in turn change the bound on $\gamma$.

Another extension is to consider the rate of convergence instead of the asymptotic convergence shown in this paper. 
\citet{srikant2019finite} derive finite-time error bounds for a stochastic approximation scheme, which they apply to TD learning with linear function approximation in the on-policy case.
Our model follows a very similar set of assumptions, so we expect that we can extend the results from Theorems~7 and 9 in~\citet{srikant2019finite} to our setting.
These results place bounds on the expected error between the iterates and the TD-fixed point, and the higher moments of this error, respectively.
However, \citet{srikant2019finite} mostly assume a constant step size, while we assume the Robbins-Monro conditions~\citep{robbins1951stochastic}.
The authors do briefly cover diminishing step sizes in Theorem 12, but this result requires technical assumptions on the step size for small values of $t$.
Furthermore, these assumptions involve terms such as the mixing time, which depends on the structure of the Markov chain.
Thus, the main challenge will be to derive corresponding assumptions for our setting.

Finally, we can test the tightness of the bound on $\gamma$ as given in \eqref{eq:gamma_cond1}.
We have performed numerical experiments to find the maximum value of $\gamma$ under which $A$ is still negative definite, which is required for convergence.
We did this for the simple random walk presented in Section \ref{sec:oned_randomwalk}, such that we can also use the tighter bound given in Corollary \ref{cor:improved_bound}.
These experiments showed a significant gap between the numerically required values for $\gamma$ and those obtained from our error bounds.
Part of this gap is likely due to the fact that Lemma \ref{lem:diag_dom} is a sufficient, but not necessary condition for $A$ to be n.d.
Thus, values of $\gamma$ that do not satisfy \eqref{eq:gamma_cond1} can still lead to convergence.
Future steps would be to derive all the terms that make up the gap, and to quantify the differences between the bounds for various settings.

%%%%%%%%%%%%%%%%%%%%%%%%%%%%%%%%%%%%%%%%%%%%%%%%%%%%%%%%%%%%%%%%
%% Bibliography
%%%%%%%%%%%%%%%%%%%%%%%%%%%%%%%%%%%%%%%%%%%%%%%%%%%%%%%%%%%%%%%%
\newcommand{\newblock}{}
\bibliography{references.bib}
\bibliographystyle{unsrtnat}

\appendix

\section{Proof of Lemma 1} \label{app:proof_diag_dom} %~\ref{lem:diag_dom}
% \begin{proof}[Proof of Lemma~\ref{lem:diag_dom}]
In this section, we prove Lemma \ref{lem:diag_dom}, which specifies the sufficient condition for $A$ to be negative definite.
In the proof we first show that we only need to consider the inner part of the matrix without the feature vectors, and subsequently analyse the symmetric counterpart of this matrix.
We derive a condition for positive definiteness of the symmetric matrix using diagonal dominance, which implies negative definiteness of matrix $A$.

Let \(B := \hat{Q}-\gamma \hat{Q}P\), so that \(A= \Phi^T (-B) \Phi\). Matrix $A$ is n.d. if 
\begin{equation}
    y^T \Phi^T (-B) \Phi y < 0,
\end{equation}
for all $y \neq 0$.
Since $\Phi$ has full column rank $y \neq 0$ if and only if $x := \Phi y \neq 0$ we
can simplify the above condition to
\begin{equation}
    x^T B x > 0,
\end{equation}
for all $x \neq 0$, i.e., it is sufficient to show that $B$ is positive definite. Now, we rewrite \(x^T B x\) as 
\begin{align}
    x^T B x &= 0.5 x^T (B + B^T) x + 0.5 x^T (B - B^T) x \\
    % &= 0.5 x^T (B + B^T) x + 0.5 x^T B x - 0.5 x^T B^T x, \\
    &= 0.5 x^T (B + B^T) x \\
    &= x^T D x,
\end{align}
where \(D:=\hat Q - 0.5 \gamma (\hat{Q}P + P^T \hat{Q})\). We need to establish that \(D\) is positive definite. It suffices to show that $D$ is Hermitian, strictly diagonal dominant, and has real positive diagonal entries~\citep[see, for instance][Corollary~7.2.3]{horn1985matrix}. 
    Matrix $D$ is Hermitian as its symmetric and real-valued.
    For the diagonal entries of $D$ we have
    \begin{equation}
        d_{ii} = \hat{q}_i -0.5 \gamma (\hat{q}_i p_{ii} + p_{ii} \hat{q}_i) = (1 - \gamma p_{ii}) \hat{q}_i > 0, 
    \end{equation}
    as $\gamma \in (0, 1)$ and $\hat q_i \geq 0$.
    A matrix is strictly diagonally dominant if 
    \begin{equation} \label{eq:diagonal_dom_condition_proof}
        |d_{ii}| - \sum_{j\neq i} |d_{ij}| > 0.
    \end{equation}
    For row $i$ of matrix $D$ we obtain
    \begin{align}
        |d_{ii}| - &\sum_{j\neq i} |d_{ij}| \\
        &= |(1 - \gamma p_{ii}) \hat{q}_i| - \sum_{j\neq i} |0.5 \gamma (\hat{q}_i p_{ij} + \hat{q}_j p_{ji})| \\
        % &= (1 - \gamma p_{ii}) \hat{q}_i - \sum_{j\neq i}0.5 \gamma (\hat{q}_i p_{ij} + \hat{q}_j p_{ji}) \\
        &= (1 - \gamma p_{ii}) \hat{q}_i - 0.5\gamma \left(\sum_{j\neq i} \hat{q}_i p_{ij} + \sum_{j\neq i} \hat{q}_j p_{ji}\right) \\
        % &= \hat{q}_i - 0.5 \gamma \left(p_{ii} \hat{q}_i + \sum_{j\neq i} \hat{q}_i p_{ij} + p_{ii} \hat{q}_i + \sum_{j\neq i} \hat{q}_j p_{ji}\right) \\
        &= \hat{q}_i - 0.5\gamma \left(\sum_{j=1}^n \hat{q}_i p_{ij} + \sum_{j=1}^n \hat{q}_j p_{ji}\right) \\
        % &= \hat{q}_i - 0.5\gamma \left(\hat{q}_i + \sum_{j=1}^n \hat{q}_j p_{ji}\right) \\
        &= \hat{q}_i\left(1-0.5\gamma \left(1 + \frac{1}{\hat{q}_i} \sum_{j=1}^n \hat{q}_j p_{ji} \right)\right). \label{eq:diagonal_dom_rewrite}
    \end{align}
    Now~\eqref{eq:diagonal_dom_condition_proof} follows from our condition on $\gamma$ and~\eqref{eq:diagonal_dom_rewrite}, concluding the proof. \qed  

\begin{myremark}
    Our proof shows that if $\gamma \hat{Q}P - \hat{Q}$ is n.d., then matrix $A$ is also n.d. We note that the reverse is not true.
    Depending on the choice of features $\Phi$, $A$ can still be n.d. if $\gamma \hat{Q}P - \hat{Q}$ is not.
\end{myremark}

\section{Proof of Theorem 2} \label{app:proofthmmain}%~\ref{thm:main}} 
In this section, we prove Theorem \ref{thm:main} for which we use the stochastic approximation result of Theorem \ref{thm:stochastic_approx}.
We first introduce two lemmas that cover the technical conditions of this theorem.
As before, we first consider the stochastic process \(X_t=(\hat S_t, S_t')\) describing the updating of states.
We have that $\mathbb{P}(\hat{S}_{t+1}=i'|\hat{S}_t=i)=\hat{p}_{ii'}$ and $\mathbb{P}(S_t'=j|\hat{S}_t=i)=p_{ij}$.
As a result, we have 
\begin{equation}
    \mathbb{P}(X_{t+1}=(i',j')|X_t=(i,j))=\hat{p}_{ii'}p_{i'j'},
\end{equation}
such that $X_t$ is a Markov chain.
Finally, we define $X=(\hat{S},S')$ as the random variable that is distributed according to the steady-state distribution of $X_t$.
This steady-state distribution equals
\begin{equation}
    \mathbb{P}(X=(i,j))=\hat{q}_i p_{ij}.
\end{equation}

We first provide a lemma to study the steady-state behaviour of matrix $A(X)$ and vector $b(X)$.
It is similar to Lemmas 7 and 8 of~\citet{tsitsiklis1997analysis}.
\begin{mylem} \label{lem:stationary_expectation}
    For $A(\cdot)$ and $b(\cdot)$ as defined in \eqref{eq:transient_a} and \eqref{eq:transient_b} we have
    \begin{equation} \label{eq:matrix_a}
        A := \mathbb{E}[A(X)] = \Phi^T (\gamma \hat{Q}P - \hat{Q})\Phi,
    \end{equation}
    and 
    \begin{equation} \label{eq:vector_b}
        b := \mathbb{E}[b(X)] = \Phi^T \hat{Q} r.
    \end{equation}
\end{mylem}
\begin{proof}
    For any $V, \Bar{V} \in \mathbb{R}^n$ we have
    \begin{align}
        \mathbb{E}[V(\hat{S}) \Bar{V}(S')] &= \sum_{i\in\mathcal{S}} \sum_{j\in\mathcal{S}} \hat{q}_i p_{ij} V(i) \Bar{V}(j)
        % \sum_{i\in\mathcal{S}} \sum_{j\in\mathcal{S}} \mathbb{P}(\hat{S}=i, S'=j) V(i)\Bar{V}(j) \\       
        % &= \sum_{i\in\mathcal{S}} \sum_{j\in\mathcal{S}} \hat{q}_i p_{ij} V(i) \Bar{V}(j) \\
        % &= \sum_{i\in\mathcal{S}} \hat{q}_i V(i) [P\Bar{V}](i) \\
        = V^T \hat{Q}P \Bar{V}.
    \end{align}
    By choosing $V=\Phi w$ and $\Bar{V} = \Phi \Bar{w}$ we get
    \begin{equation}
        \mathbb{E}[w^T\phi(\hat{S}) \phi^T(S')\Bar{w}] = w^T\Phi^T\hat{Q}P\Phi \Bar{w}.
    \end{equation}
    As we can choose vectors $w$ and $\Bar{w}$ as we like, we must also have
    \begin{equation}
        \mathbb{E}[\phi(\hat{S}) \phi^T(S')] = \Phi^T\hat{Q}P\Phi.
    \end{equation}
    By analogy, we find
    \begin{equation}
        \mathbb{E}[\phi(\hat{S}) \phi^T(\hat{S})] = \Phi^T\hat{Q}\Phi,
    \end{equation}
    where now we do not consider the successor state $S'$ so matrix $P$ disappears.
    We use both equations to obtain
    \begin{align}
        \mathbb{E}[A(\hat S, S')]
        &= \mathbb{E}[\phi(\hat{S})\left(\gamma \phi^T(S') - \phi^T(\hat{S})\right)] \\
        &= \gamma \mathbb{E}[\phi(\hat{S})\phi^T(S')]  - \mathbb{E}[\phi(\hat{S}) \phi^T(\hat{S})] \\
        &= \Phi^T (\gamma \hat{Q} P - Q)\Phi.
    \end{align}
    Using a similar analysis, we also have 
    \begin{equation}
        \mathbb{E}[b(\hat S, S')] = \mathbb{E}[\phi(\hat{S})r(\hat{S})] = \Phi^T \hat{Q} r.
    \end{equation} %\qed
\end{proof}

% \section{Proof of Lemma~\ref{lem:cond56}}
\begin{mylem} \label{lem:cond56}
    Conditions 5 and 6 of Theorem \ref{thm:stochastic_approx} are satisfied for Markov chain $X_t$.
    \end{mylem}
\begin{proof}
    The perturbed Markov chain $\{\hat{S}_t \ | \ t=0,1,\ldots\}$ is irreducible, aperiodic, and finite.
    As a result, we know that there exists scalars $O > 0$ and $\beta \in (0,1)$ such that 
    \begin{equation}
        |\mathbb{P}(\hat{S}_t=i|\hat{S}_0)-\hat{q}_i| \leq O \beta^t.
    \end{equation}
    for all $\hat{S}_0\in\mathcal{S}$.
    Now fix any $\hat{S}_0\in\mathcal{S}$ and define diagonal matrices $\hat{Q}_t$ for $t=0,1,\ldots$, with $\mathbb{P}(\hat{S}_t=i|\hat{S}_0)$ as the $i$th diagonal element.
    As a result, we have 
    \begin{equation}
        ||\hat{Q}_t-\hat{Q}|| \leq O \beta^t.
    \end{equation}
    We work out the expectations of $A(\cdot)$ and $b(\cdot)$ with respect to initial state $\hat{S}_0$.
    From \eqref{eq:transient_a} and \eqref{eq:transient_b} we have
    \begin{align}
        \mathbb{E}[A&(X_t)|X_0] = \mathbb{E}[\phi(\hat{S}_t)(\gamma \phi^T(S_t') - \phi^T(\hat{S}_t))] \\
        &= \gamma \mathbb{E}[\phi(\hat{S}_t)\phi^T(S_t')|X_0] - \mathbb{E}[\phi(\hat{S}_t)\phi^T(\hat{S}_t)|X_0],
    \end{align}
    and
    \begin{equation}
        \mathbb{E}[b(X_t)|X_0] = \mathbb{E}[\phi(\hat{S}_t)r(\hat{S}_t)|X_0].
    \end{equation}
    In a similar way to Lemma \ref{lem:stationary_expectation}, we can work out the expectations of the equations above.
    For any $V, \Bar{V} \in \mathbb{R}^n$, we have
    \begin{align}
        \mathbb{E}[V(\hat{S}_t)\Bar{V}(S_t')&|\hat{S}_0] \\
        &= \sum_{i\in\mathcal{S}}\sum_{j\in\mathcal{S}}\mathbb{P}(\hat{S}_t=i, S_t'= j|\hat{S}_0) V(i) \Bar{V}(j) \\
        % &= \sum_{i\in\mathcal{S}}\sum_{j\in\mathcal{S}}\mathbb{P}(\hat{S}_t=i|S_0) \mathbb{P}(S_t'=j|\hat{S}_t=i) V(i) \Bar{V}(j) \\
        &= \sum_{i\in\mathcal{S}}\mathbb{P}(\hat{S}_t=i|\hat{S}_0)V(i) \sum_{j\in\mathcal{S}} p_{ij} \Bar{V}(j) \\
        % &= \sum_{i\in\mathcal{S}}\mathbb{P}(\hat{S}_t=i|S_0)V(i) P[\Bar{V}](i) \\ 
        &= V^T \hat{Q}_t P \Bar{V}.
    \end{align}
    By specifying $V=\Phi w$ and $\Bar{V}=\Phi \Bar{w}$ we get
    \begin{equation}
        \mathbb{E}[w^T \phi(\hat{S}_t)\phi^T(S_t')\Bar{w}|\hat{S}_0] = w^T \Phi^T \hat{Q}_t P \Phi \Bar{w}.
    \end{equation}
    and since the choice of $w$ and $\Bar{w}$ is arbitrary we have 
    \begin{equation} \label{eq:transient1}
        \mathbb{E}[\phi(\hat{S}_t)\phi^T(S_t')|\hat{S}_0] = \Phi^T \hat{Q}_t P \Phi.
    \end{equation}
    Similarly, it follows that
    \begin{equation} \label{eq:transient2}
        \mathbb{E}[\phi(\hat{S}_t)\phi^T(\hat{S}_t)|\hat{S}_0] = \Phi^T \hat{Q}_t \Phi,
    \end{equation}
    and
    \begin{equation} \label{eq:transient3}
        \mathbb{E}[\phi(\hat{S}_t)r(\hat{S}_t)|\hat{S}_0] = \Phi^T \hat{Q}_t r.
    \end{equation}
    To prove the first part of Condition 5 we use \eqref{eq:matrix_a}, \eqref{eq:transient1}, and \eqref{eq:transient2} such that 
    \begin{align}
        \sum_{t=0}^\infty ||\mathbb{E}[&A(X_t)|X_0] - A|| \\
        &= \sum_{t=0}^\infty ||\Phi^T (\hat{Q}_t - \hat{Q})(\gamma P - I) \Phi|| \\
        &\leq \sum_{t=0}^\infty K^2 \max_{k,j}|\phi^T_k (\hat{Q}_t - \hat{Q})(\gamma P - I) \phi_j| \\
        &\leq K^2 \max_k ||\phi_k||^2 \sum_{t=0}^\infty ||(\hat{Q}_t - \hat{Q})(\gamma P - I)|| \\
        &\leq K^2 \max_k ||\phi_k||^2 (\gamma G + 1)\sum_{t=0}^\infty ||\hat{Q}_t - \hat{Q}|| \\
        &\leq K^2 \max_k ||\phi_k||^2 (\gamma G + 1) \frac{O}{1-\beta},
    \end{align}
    where $G=||P||$.
    For the second part of condition 5 we use \eqref{eq:vector_b} and \eqref{eq:transient3} to obtain
    \begin{align}
        \sum_{t=0}^\infty ||\mathbb{E}[b(X_t)|X_0] &- b|| = \sum_{t=0}^\infty ||\Phi^T (\hat{Q}_t - \hat{Q}) r|| \\
        &\leq \sum_{t=0}^\infty K \max_k |\phi^T_k (\hat{Q}_t - \hat{Q}) r| \\
        &\leq K \max_k ||\phi_k|| \sum_{t=0}^\infty ||(\hat{Q}_t - \hat{Q}) r|| \\
        &\leq K \max_k ||\phi_k|| R \sum_{t=0}^\infty ||\hat{Q}_t - \hat{Q}|| \\
        &\leq K \max_k ||\phi_k|| R \frac{O}{1-\beta},
    \end{align}
    where $R=||r||$.
    Then, Condition 5 is satisfied by setting $M=K \max_k ||\phi_k|| \frac{O}{1-\beta} \max\{K \max_k ||\phi_k|| (\gamma G + 1), R\}$, $g=1$, and $h(x)=1$ for all $x$.
    As $h$ is a constant function, Condition 6 is also trivially satisfied by for example setting $\mu_g=1$ for all $g$. %\qed
    \end{proof}

    % \subsection{Proof of Theorem \ref{thm:main}} \label{sec:proof1}
    \begin{proof}[Proof of Theorem \ref{thm:main}]
    % To prove convergence of Theorem \ref{thm:main},
    We apply Theorem \ref{thm:stochastic_approx}. Therefore, we need to check all conditions.
    Condition 1 directly follows from \eqref{eq:step_size}.
    Furthermore, we observed that $X_t$ is indeed a Markov chain, while Condition 3 is proven in Lemma \ref{lem:stationary_expectation}.
    Finally, Conditions 5 and 6 follow from Lemma \ref{lem:cond56}.
    We are left with proving that $A$ is negative definite.

    Both the original and perturbed chains are reversible, such that we have $q_j = \rho_{ij} q_i$ and $\hat{q}_j = \hat{\rho}_{ij} \hat{q}_i$
    for $i,j\in\mathcal{S}$, with transition ratios as defined in \eqref{eq:rate} for the original and perturbed chains.
    Then, we define 
    \begin{equation}
        c_{ij} = \frac{\hat{\rho}_{ij}}{\rho_{ij}}
    \end{equation}
    for any pair of states $i,j\in \mathcal{S}$ that satisfy $p_{ij}>0$.
    By reversibility we then also have $p_{ji}, \hat{p}_{ij}, \hat{p}_{ji}>0$ so $c_{ij}$ is well defined.
    We set
    \begin{equation}
        c = \max_{i,j\in\mathcal{S}}\left\{\max\left\{c_{ij}, \frac{1}{c_{ij}}\right\}\right\}.
    \end{equation}
    As a result, we obtain
    \begin{align} \label{eq:proof_balance}
        \sum_{j=1}^n \hat{q}_j p_{ji} &= \sum_{j=1}^n \hat{q}_i \hat{\rho}_{ij} p_{ji} = \sum_{j=1}^n \hat{q}_i c_{ij} \rho_{ij} p_{ji} \\
        &= \sum_{j=1}^n \hat{q}_i c_{ij} p_{ij} = \hat{q}_i \sum_{j=1}^n c_{ij} p_{ij} \\
        &\leq \hat{q}_i c \sum_{j=1}^n p_{ij} = \hat{q}_i c.
    \end{align}
    Finally, we have
    \begin{equation}
        \frac{2}{1+\frac{1}{\hat{q}_i}\sum_{j=1}^n \hat{q}_j p_{ji}} \geq \frac{2}{1+\frac{1}{\hat{q}_i}c \hat{q}_i}  = \frac{2}{c+1} > \gamma,
    \end{equation}
    which proves that $A$ is n.d. by Lemma \ref{lem:diag_dom}.
    As a result, we can apply Theorem \ref{thm:stochastic_approx} to show that $w_t$ converges to $w^*$ and 
    \begin{align}
        &A w^* + b = 0, \\
        &\Phi^T \hat{Q} (r + \gamma P \Phi w^* - \Phi w^*) = 0,
    \end{align}
    where we use \eqref{eq:matrix_a} and \eqref{eq:vector_b} for $A$ and $b$.
    We use the latter of these two equation such that the point of converge $w^*$ satisfies
    \begin{align}
        \text{PBE}&(w^*)= ||\Pi_{\hat{Q}}\mathcal{B}(\Phi w^*) - \Phi w^*||_{\hat{Q}} \\
        &= ||\Phi\left(\Phi^T \hat{Q} \Phi\right)^{-1}\Phi^T \hat{Q}\left(r+\gamma P \Phi w^*\right) - \Phi w^*||_{\hat{Q}} \\
        &= ||\Phi\left(\Phi^T \hat{Q} \Phi\right)^{-1}\Phi^T \hat{Q} \Phi w^* - \Phi w^*||_{\hat{Q}} \\
        &= ||\Phi w^* - \Phi w^*||_{\hat{Q}} = 0.
    \end{align}
    This proves Theorem \ref{thm:main}. 
\end{proof}

\section{Proof of Corollary 1} \label{app:proof_cor_improved_bound} %~\ref{cor:improved_bound}} 
This section contains the proof of Corollary~\ref{cor:improved_bound} for the improved bound on convergence for one-dimensional random walks with constant transition ratios $\rho$ and $\hat{\rho}$. 
% \begin{proof}
Given these transition ratios, we have $p_{i,i+1}=p_{nn}=\frac{\rho}{\rho+1}$ for $i=1,\ldots,n-1$.
and $p_{i-1,i}=p_{11}=\frac{1}{\rho+1}$ for $i=2,\ldots,n$, while all other transitions have probability 0.
If we pick up the proof of Theorem \ref{thm:main} from~\eqref{eq:proof_balance} we obtain
\begin{align} \label{eq:simple_rw_cases}
    \sum_{j=1}^n \hat{q}_j p_{ji} &= \hat{q}_i \sum_{i=1}^n c_{ij} p_{ij} \\
    &= \begin{cases}
        \hat{q}_1 (\frac{1}{\rho+1} + \delta\frac{\rho}{\rho+1}) &\mbox{ if } i=1, \\
        \hat{q}_i (\frac{1}{\delta}\frac{1}{\rho+1} + \delta\frac{\rho}{\rho+1}) &\mbox{ if } i=2,\ldots,n-1, \\
        \hat{q}_n (\frac{1}{\delta}\frac{1}{\rho+1} + \frac{\rho}{\rho+1}) &\mbox{ if } i=n,
    \end{cases}
\end{align}
where we use that $c_{i,i+1}=\delta=\frac{\hat{\rho}}{\rho}$, $c_{i+1,i}=\frac{1}{\delta}$ for $i=1,\ldots,n-1$, and $\rho_{11}=\hat\rho_{11}=c_{11}=\rho_{nn}=\hat\rho_{nn}=c_{nn}=1$.
To prove convergence, we apply Lemma \ref{lem:diag_dom} for which we need to satisfy \eqref{eq:diag_dom_condition} for all $i=1,\dots,n$.
For the simple random walk with constant transition ratios this condition then comes down to
\begin{equation}
    \gamma < \min\left\{\frac{2(\rho+1)}{\rho + 2 + \delta\rho}, \frac{2(\rho+1)}{\rho+1+\frac{1}{\delta} + \delta\rho}, \frac{2(\rho+1)}{2\rho+1+\frac{1}{\delta}}\right\},
\end{equation}
by filling in each of the cases of \eqref{eq:simple_rw_cases}. It is easy to see that for \(\delta\geq 1\) the first case is the smallest, and that for \(\delta\leq 1\) the third case is the smallest. Therefore, we can leave out the second case from the minimization. \qed

\end{document}